\documentclass[11pt]{article}
\usepackage[sf]{titlesec}
\usepackage{graphicx,wrapfig}
\usepackage{setspace}
\usepackage{booktabs}
\usepackage{amsmath}
\usepackage{comment}

\usepackage[ruled,vlined]{algorithm2e}

\usepackage{svg}
\usepackage{amsthm}
\usepackage{mathtools}

\usepackage[utf8]{inputenc}
\usepackage{tikz}

\usetikzlibrary{
   arrows.meta, 
   bending
}

\usepackage{hyperref}
\usepackage{scalerel,amssymb}
\usetikzlibrary{shapes,arrows}
\usepackage{graphics} 
\usepackage{epsfig} 
\usepackage{subcaption}
\usepackage[square,numbers]{natbib}
\usepackage{authblk}
\usepackage[margin=0.6in]{geometry}

\usepackage{float}

\newtheorem{prop}{Proposition}

\restylefloat{figure}

\providecommand*{\ped}[1]{\ensuremath{_\mathrm{#1}}}

\usepackage{enumitem}
\setlist[itemize]{noitemsep}

\titlespacing\section{0pt}{12pt plus 2pt minus 2pt}{0pt plus 2pt minus 2pt}
\titlespacing\subsection{0pt}{12pt plus 2pt minus 2pt}{0pt plus 2pt minus 2pt}
\titlespacing\subsubsection{0pt}{12pt plus 2pt minus 2pt}{0pt plus 2pt minus 2pt}

\usepackage{scalefnt}
\usepackage{lipsum}

\providecommand*{\diff}%
   {\@ifnextchar^{\DIfF}{\DIfF^{}}}
\def\DIfF^#1{%
   \mathop{\mathrm{\mathstrut d}}%
      \nolimits^{#1}\gobblespace
}
\def\gobblespace{%
   \futurelet\diffarg\opspace}
\def\opspace{%
   \let\DiffSpace\!%
   \ifx\diffarg(%
      \let\DiffSpace\relax
   \else
      \ifx\diffarg\[%
         \let\DiffSpace\relax
      \else
         \ifx\diffarg\{%
            \let\DiffSpace\relax
         \fi\fi\fi\DiffSpace}
\makeatother

\title{Loiter UAV Reinsertion Guidance for Fixed-wing UAV Corridors\footnote{Copyright © 2026 by Pradeep J, Kedarisetty Siddhardha, and Ashwini Ratnoo. Published by the American Institute of Aeronautics and Astronautics, Inc., with permission.} \footnote{Presented in AIAA SCITECH 2026 Forum, Publised version DoI: https://doi.org/10.2514/6.2026-1982}}

\author{Pradeep J\footnote{B.Tech (Honours) Student, Department of Mechatronics Engineering;pradeepj0406@gmail.com} and Kedarisetty Siddhardha \footnote{Assistant Professor, Department of Mechatronics Engineering; siddhardhak@iitbhilai.ac.in}}
\affil{Indian Institute of Technology Bhilai, Bhilai, 491002, India}
\author{Ashwini Ratnoo\footnote{Professor, Department of Aerospace Engineering; ratnoo@iisc.ac.in. Associate Fellow AIAA.}}
\affil{Indian Institute of Science, Bengaluru, 560012, India}

\date{\vspace{-7ex}}

\begin{document}

\maketitle

\begin{abstract}
This paper considers fixed-wing unmanned aerial vehicle (UAV) corridors comprising a main lane, a circular loiter lane for managing traffic congestion, and transit lanes connecting the two.
In particular, we address the problem of conflict-free reinsertion of UAVs from the loiter lane back into the main lane.
The loiter lane contains a fixed number of equidistant virtual slots that UAVs can occupy.
Reinsertion of loiter UAVs into the main lane becomes essential either due to reduced traffic in the main lane or due to a loiter UAV needing to reach its destination urgently.
Given the total number of loiter slots, UAV speed limits, and the minimum safety distance, a guidance algorithm is developed to compute the required speed of a loiter UAV in the transit lane to ensure safe reinsertion.
The proposed guidance and automation strategies are validated through numerical simulations.

\end{abstract}

\section{Introduction}

The widespread use of UAVs for applications such as communications~\cite{ghamari2022unmanned}, defence~\cite{siddhardha2023intercepting}, and goods delivery~\cite{betti2024uav} is accelerating due to steady advancements in technologies like novel configurations~\cite{siddhardha2018novel}, high endurance batteries~\cite{xiao2023design}, and advanced guidance algorithms~\cite{tsalik2025n}.
In this context, traffic management algorithms that resolve conflicts and ensure smooth UAV traffic flow are critically important.
To enable congestion-free operations, several decentralized conflict resolution techniques have been developed, including congestion-avoidance~\cite{maniccam2006adaptive}, dynamic traffic routing~\cite{badrinath2019integrated}, and backpressure algorithms~\cite{chin2023traffic}.
In parallel, structured airspace protocols based on virtual corridors and lane-based traffic management systems are being formulated to enable safe and organized UAV operations in Class-G airspace~\cite{yadav2021uav}.
Notable emerging UAV traffic management (UTM) systems include NASA UTM~\cite{prevot2016uas}, SESAR U-space~\cite{huttunen2019u}, Corridrone~\cite{tony2020corridrone}, and EuroDRONE~\cite{lappas2022eurodrone}.
A key challenge in designing such UTMs is preventing traffic congestion within the corridor.
This is especially critical for corridors serving fixed-wing UAVs, which require a minimum forward speed to maintain flight.

While corridor-based architectures help organize UAV movement, it is crucial for both UTMs and individual UAVs to use algorithms that ensure collision-free operation.
Geofencing is one such technique that prevents collisions by surrounding objects with virtual boundaries~\cite{tony2021lane}.
However, geofencing mainly sets safety rules for UAV navigation and does not help reduce traffic congestion or manage changes in traffic flow.
Corridor congestion can arise due to limited landing or approach slots, intersecting corridors, or adverse weather conditions~\cite{bhise2022signed}.
In these situations, fixed-wing UAVs need to use loitering strategies to safely manage traffic and maintain flow through the corridor.

A few studies have demonstrated the value of loitering strategies in UAV applications.
Schouwenaars et al.\cite{schouwenaars2004receding} proposed a path-planning method for fixed-wing UAVs that uses loitering to hold position while navigating cluttered environments.
Wilhelm et al.~\cite{wilhelm2016development, wilhelm2017direct} developed efficient loitering strategies for surveillance, enabling smooth, tangential entry into loiter paths around multiple targets and minimizing travel distance.
However, these approaches assume unoccupied loiter paths and do not address safe, coordinated use among multiple UAVs.
In our previous work~\cite{kedarisetty2023cooperative, kedarisetty2025loiter}, we introduced the concept of loiter lanes within a fixed-wing UAV corridor to manage traffic in the main lane.
Guidance algorithms were developed to divert UAVs from the main lane to the loiter lane in a conflict-free manner.

In this work, we address the complementary problem of reinserting a loitering UAV back into the main lane.
First, it is shown that in low traffic density and sufficient gaps between main lane UAVs, the outgoing loiter UAV can reenter the main lane without changing the positions of the main lane UAVs.
Next, the condition necessary for an outgoing UAV to be inserted into the main lane without a sufficient gap between main lane UAVs is presented.
Based on this condition, the corridor geometry and the guidance and automation algorithms that enable the insertion of a loiter UAV into the main lane are developed.
Finally, the effectiveness of the proposed algorithms is demonstrated through numerical simulations.

\section{Scenario and Problem statement}
Consider a fixed-wing UAV corridor with a main lane and a proposed loiter path, as shown in Fig.~\ref{fig:concept}.
The loiter path includes two transit lanes, a loiter circle, and two transit links.
The transit lanes and links serve as intermediate paths between the main lane and loiter circle, allowing for smooth entry and exit of the UAVs.
The transit links are circular arcs connecting the main lane to the transit lanes.
Fixed-wing UAVs flying in the corridor are considered to operate within speed limits of $V\ped{min}$ and $V\ped{max}$, and the radii of the transit links are larger than the UAVs' minimum turn radius.
\begin{figure}[H]
    \centering
    \includegraphics[width=0.4\linewidth]{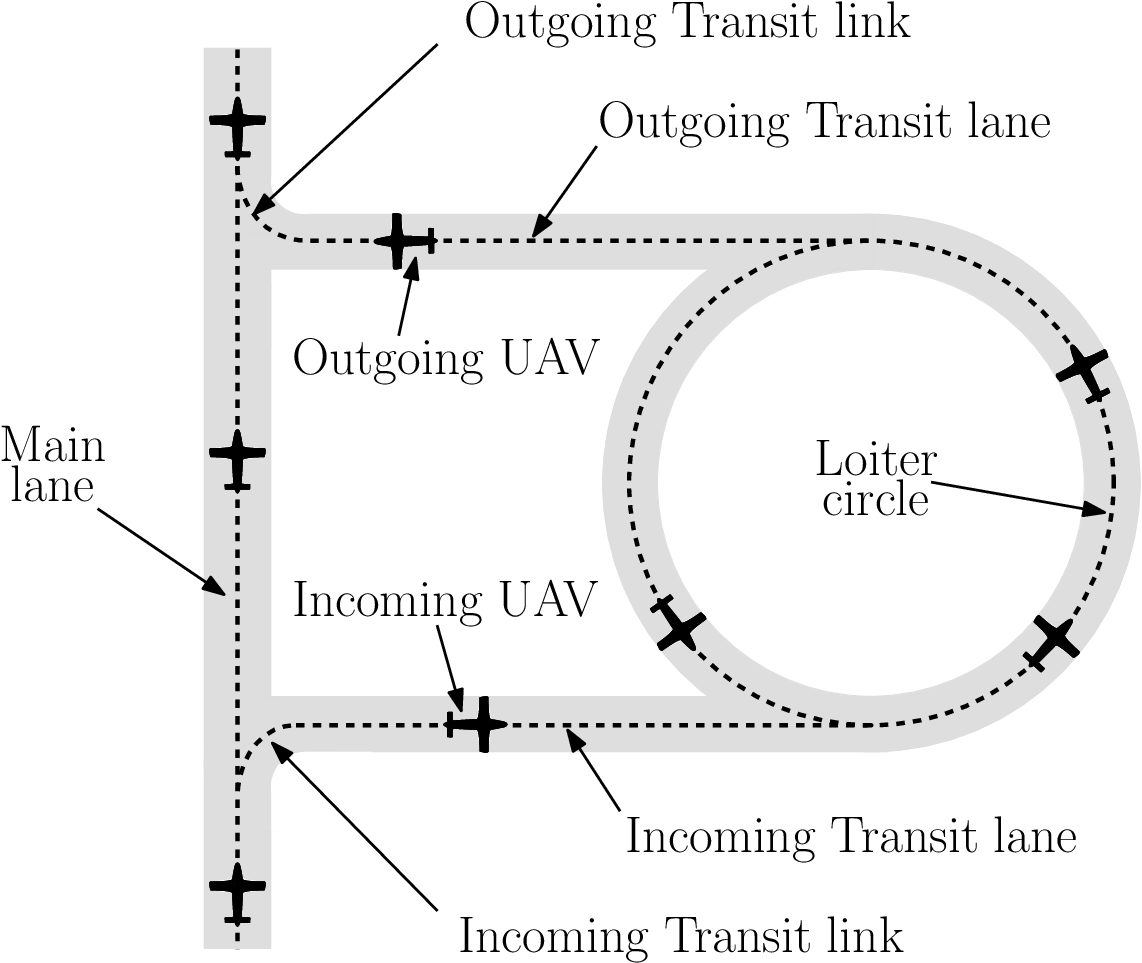}
    \caption{Loiter Lane Concept in Fixed-wing UAV Corridors}
    \label{fig:concept}
\end{figure}
\vspace{-0.4cm}
\begin{figure}[H]
    \centering
    \begin{subfigure}[b]{0.5\textwidth}
        \centering
\includegraphics[width=5.2cm]{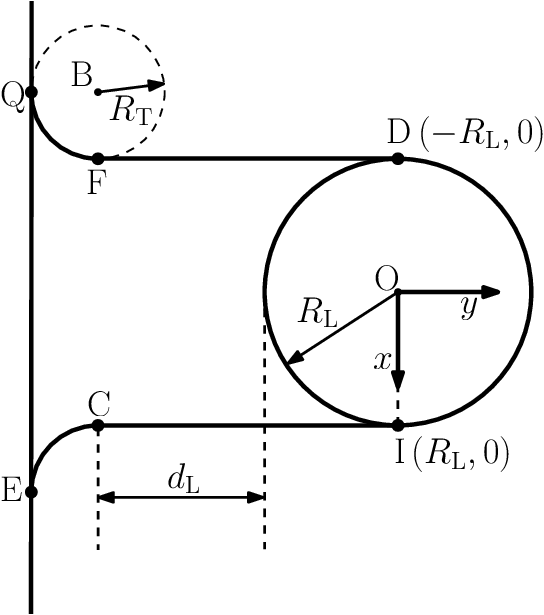}
\caption{Loiter path description}\label{fig:corridor}
    \end{subfigure}%
\hfill
    \begin{subfigure}[b]{0.5\textwidth}
        \centering
\includegraphics[width=6.2cm]{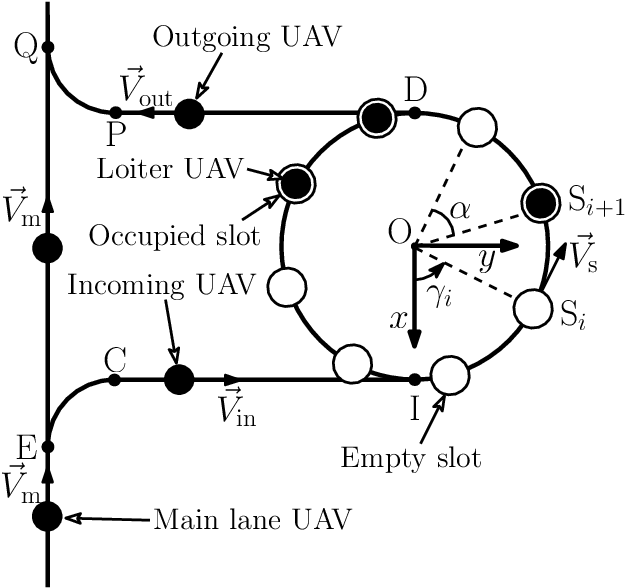}
\caption{Virtual slots in loiter circle}\label{fig:virtual_slots}
    \end{subfigure}%
     \caption{Corridor description and geometry of proposed loiter lane}
    \label{fig:dir_app}
\end{figure}

The details of the proposed loiter path and the corridor geometry are shown in Fig.~\ref{fig:corridor}.  
The diversion from the main lane starts at point E. 
An incoming UAV follows a circular path from E (transit link path) to enter the transit lane at C. 
The transit lane is a straight line path connecting C and point I, which is the point of entry to the loiter lane. 
FD and $\overset{\frown}{FQ}$ are the outgoing transit lane and transit link, respectively.
The centers and radii of the loiter lane and outgoing transit link path are (O, $R\ped{L}$) and (A, $-R\ped{T}$), respectively.
An Earth-fixed reference frame is placed at the center of the loiter lane (O).
Here, $d\ped{L}$ represents the lateral separation between point A and the loiter lane.

We introduce the concept of virtual slots in the loiter lane, as illustrated in Fig.~\ref{fig:virtual_slots}.
Each virtual slot $\mathrm{S}_i$, where $i \in [1,\,N]$, represents an imaginary position along the loiter lane that a UAV may occupy, provided it is unoccupied.
The loiter lane is characterized by the following properties:
a) the angular separation $\alpha$ between consecutive virtual slots is uniform, and
b) all virtual slots move at a constant speed $V\ped{s}$, where $V\ped{s} \in \left[V\ped{min},\, V\ped{max}\right]$.

\noindent \textbf{Problem Statement:} a) Compute the speed of outgoing UAV $V\ped{out} \in [V\ped{min},\,V\ped{max}]$ in the outgoing transit lane such that it can enter the main lane by maintaining minimum safety distance ($d\ped{s})$ with all other UAVs in the corridor, and b) compute the main lane UAVs' ($V_{\mathrm{m}_i}$) speeds for creating space in the main lane for outgoing UAV insertion. 

\section{Fixed-wing UAV Corridor Design}
The design for the proposed fixed-wing UAV corridor structure involves determining the loiter radius $R\ped{L}$ and loiter separation $d\ped{L}$.

\begin{prop}
For $V\ped{s}\leq V\ped{in}$, the choice of the loiter lane radius $R\ped{L} = \dfrac{d\ped{s}}{2 \sin^2 \left( \dfrac{\pi}{N} \right)}$ accommodates $N$ ($\geq 2$) equiangular virtual slots with a minimum safe distance $d\ped{s}$ between all UAVs. 
\end{prop}
\begin{proof} The proof for this proposition is provided in detail in our previous work~\cite{kedarisetty2025loiter}.
\end{proof}

\begin{prop}
    For a given $R\ped{L}$, $R\ped{T}$, $V\ped{min}$, and $V\ped{max}$, the loiter separation is given as
    \begin{equation}
        d\ped{L} = \dfrac{\Delta d\ped{P}}{V\ped{m}\left( \dfrac{1}{V\ped{min}} - \dfrac{1}{V\ped{max}}\right)} - \dfrac{\pi R\ped{T}}{2} - R\ped{L} \label{eq:1}
    \end{equation}
    where $\Delta d\ped{P}$ is the main lane patch length in which the outgoing UAV is to be inserted.
\end{prop}
\begin{proof}
    Consider the fixed-wing UAV corridor geometry presented in Fig.~\ref{fig:corridor_patch}. 
    For the outgoing UAV at point D to rendezvous with a point P on the main lane at point Q, the following condition must be satisfied
\begin{equation}
    \dfrac{d\ped{P}}{V\ped{m}} = \dfrac{\dfrac{\pi R\ped{T}}{2} + d\ped{L} + R\ped{L}}{V\ped{out}}
\end{equation}
where $d\ped{P}$ is the distance between points P and Q, and the speed of point P is assumed to be $V\ped{m}$.
The minimum and maximum distances of the main lane that can be reached by the UAV, abiding by its speed limits, are
\begin{equation}
    d\ped{P_{min}} = \dfrac{V\ped{m}}{V\ped{max}}\left(\dfrac{\pi R\ped{T}}{2} + d\ped{L} + R\ped{L}\right), \quad d\ped{P_{max}} = \dfrac{V\ped{m}}{V\ped{min}}\left(\dfrac{\pi R\ped{T}}{2} + d\ped{L} + R\ped{L}\right) \label{eq:3}
\end{equation}
Accordingly, the length of the patch in which the outgoing UAV can be inserted is $\Delta d\ped{P} = d\ped{P_{max}} - d\ped{P_{min}}$, which can be expressed from Eq.~\eqref{eq:3} as
\begin{equation}
    \Delta d\ped{P} = V\ped{m}\left(\dfrac{\pi R\ped{T}}{2} + d\ped{L} + R\ped{L}\right)\left( \dfrac{1}{V\ped{min}} - \dfrac{1}{V\ped{max}} \right) \label{eq:4}
\end{equation}
Rearranging Eq.~\eqref{eq:4}, yields the relation between $d\ped{L}$ and $\Delta d\ped{P}$, as given in Eq.~\eqref{eq:1}.
\end{proof}

\section{Reinsertion Guidance and Automation}
This section develops the guidance and automation algorithms for inserting the outgoing UAV into the feasible patch of insertion on the main lane, which is indicated as a green strip in Fig.~\ref{fig:corridor_patch}.
The main lane UAVs in this feasible patch are denoted as $\mathrm{M}_i$, where $i\in[1,\,m]$.
The distance between adjacent UAVs is denoted by $d_{j\, j+1}$ with $j\in[1,\,m-1]$, and the distances from the start of the strip to the first UAV and the last UAV to the end of the strip are denoted as $d_{0,1}$ and $d_{m\,m+1}$, respectively.

The necessary condition for inserting the outgoing UAV at point D into the main lane is 
\begin{equation}
    \Delta d\ped{P} \geq m d\ped{s}
 \end{equation}
If this condition is violated, it implies that there does not exist sufficient space in the feasible patch for the outgoing UAV to be inserted and the outgoing UAV continue to loiter.
If this condition is satisfied, then two further possibilities arise:  
(a) the gap between any adjacent UAVs is greater than $2d\ped{s}$, or either $d_{0,1}$ or $d_{m,m+1}$ is greater than $d\ped{s}$; and  
(b) the gap between all adjacent UAVs is less than $2d\ped{s}$, and both $d_{0,1}$ and $d_{m,m+1}$ are less than $d\ped{s}$. 

In the first case, the solution is straightforward—the outgoing UAV selects the first feasible gap in the strip. In the second case, the UAVs in the green patch are assigned speed commands to adjust their velocities from $V\ped{m}$ to $V\ped{max}$ or $V\ped{min}$ as needed to create sufficient space. The UAV speed in main lane are denoted by $V_{i}$ where $i\in[1,m]$. Algorithm~\ref{alg:time} computes the reinsertion time $t_\mathrm{out}$ and determines the appropriate patch section for reinsertion. The patch is divided into segments and $h$ index of the segment. The outgoing and main lane UAV speeds are determined using Algorithm \ref{alg:V_speed}.

\begin{figure}[H]
    \centering
    \includegraphics[width=0.4\linewidth]{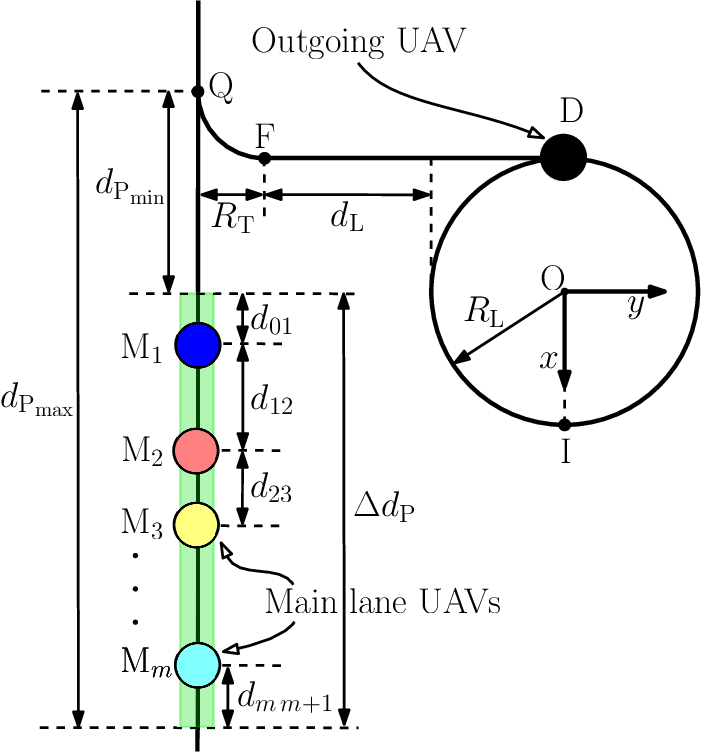}
    \caption{Fixed-wing UAV corridor geometry}
    \label{fig:corridor_patch}
\end{figure}

\begin{algorithm}[H]
\caption{Algorithm to determine the reinsertion time of an outgoing UAV and the section to which it should be inserted }
\label{alg:time}
\SetKwInOut{KwIn}{Input}
\SetKwInOut{KwOut}{Output}
\SetKwInOut{KwPro}{Procedure}
\SetKwInOut{KwIni}{Initialize}

\KwIn{$m$,\,$d\ped{P_{min}}$,\,$d\ped{P_{max}}$,\,$d\ped{s}$,\,$d_{i,i+1}$,\,$d\ped{P}$}
\KwOut{$t\ped{out}$, \, $h$} 
\KwPro{}

\If{$\Delta d_\mathrm{P}<m\cdot d_\mathrm{s}$}{
 \textbf{return} $t_\mathrm{out}==\mathrm{NULL}$ 
 }
 
\For{$i = 0$ \KwTo $m$}
{ 
    \If{(i==0)}{
    \If{$d_{i,i+1}>d_\mathrm{s}$}{ $t_\mathrm{out}=\frac{d_\mathrm{P_{min}}}{V_\mathrm{m}}$ 
        }
    }
    \If{(i==m)}{
    \If{$d_{i,i+1}>d_\mathrm{s}$}{ $t_\mathrm{out}=\frac{d_\mathrm{P_{max}}}{V_\mathrm{m}}$ 
        }
    }
    \Else{
    \If{$d_{i,i+1}>2d_\mathrm{s}$}{ $t_\mathrm{out}=\frac{d_\mathrm{i}+d_\mathrm{i+1}}{2V_\mathrm{m}}$ 
    }
    }
}

\If{$t_\mathrm{out}==\mathrm{none}$}{
     $h=\mathrm{find}\left( \max(d_{i,i+1}) \right)$,  {$0 \le i \le m$}

     $t_\mathrm{out}=\frac{d_\mathrm{P_{max}}-((m+1)-h)\cdot d_\mathrm{s}}{V_\mathrm{m}}$    
}

\Return $t_\mathrm{out}$,$h$
\end{algorithm}

If the speeds of the main-lane UAVs need to be adjusted, the main-lane UAVs will compress to maintain the safety distance $d_\mathrm{s}$, providing sufficient space in the $h$-th segment for outgoing UAV to insert safely. The outgoing UAV can also be inserted at the beginning or the end of the patch.

\begin{algorithm}[H]
\caption{Algorithm for calculating all UAV speeds}
\label{alg:V_speed}
\SetKwInOut{KwIn}{Input}
\SetKwInOut{KwOut}{Output}
\SetKwInOut{Procedure}{Procedure}
\SetKwInOut{Initialize}{Initialize}

\KwIn{$m$,\,$d_\mathrm{P_{min}}$,\,$d_\mathrm{P_{max}}$,\,$d_\mathrm{s}$,\,$t_\mathrm{out}$,\,$h$,\,$d_{i,i+1}$}
\KwOut{$V_\mathrm{out}$,\,$V_i$} 

\Procedure{}
 $V_\mathrm{out}=\dfrac{\dfrac{\pi R\ped{T}}{2} + d\ped{L} + R\ped{L}}{t_\mathrm{out}}$ 
\For{$i = 1$ \KwTo $m$}{
    \If{$\mathrm{isEmpty}{(h)}$}{
    $V_{i}=V_\mathrm{m}$
    }
    \Else{
        \If{$i  < h $}{ 
            \If{$(d_{i,i-1}>d_\mathrm{s} \, \, \&\& \, \,  i>1)||(d_{i,i-1}>0 \, \, \&\& \, \, i==1)$}{
            $V_{i} = V_\mathrm{max}$ } 
            }
        
        \If{$i  \geq h$} {
            \If{$(d_{i,i+1}>d_\mathrm{s} \, \, \&\& \, \,  i<m)||(d_{i,i+1}>0 \, \, \&\& \, \, i==m)$}{
            $V_{i} = V_\mathrm{min}$ } 
        }
        \If{$V_{i}==\mathrm{none}$}{
         $V_{i} = V_\mathrm{m}$ 
        }
    }
}

\Return $V_\mathrm{out}$,$V_{_i}$
\end{algorithm}

\section{Simulation results}

The proposed algorithm are validated through simulations by a modified unicycle model with variable speed, as shown in Eq.~\eqref{eq:model}, to emulate the dynamics of a fixed-wing UAV.
\begin{equation}
    \begin{aligned}
    \dot{x} &= V \cos\theta, \quad 
    \dot{y} = V \sin\theta, \quad
    a = V \dot{\theta}
    \label{eq:model}
    \end{aligned}
\end{equation}
where $(x,\,y)$ is the UAV position, $V$ is its speed, $a$ is the lateral acceleration, and $\theta$ is the flight path angle. $V$ and $a$ are control input to the UAVs and it is provided by the guidance algorithm.
Two simulation cases were considered: (a) a scenario in which a feasible $d_{i,i+1}$ exists, meaning that one section of the patch is suitable for insertion, and  
(b) a scenario in which no feasible $d_{i,i+1}$ is available initially. The parameters chosen for the simulation are listed in Table \ref{table:params}.

\begin{table}[h]
\centering
\caption{Simulation Parameters}
\begin{tabular}{c c c c c c c c c c c}
\hline
\hline
$V_{\min}$ & $V_{\max}$ & $R\ped{L}$ & $R\ped{T}$ & $d\ped{L}$ & $N$ & $m$ & $d\ped{s}$ & $\Delta d_\mathrm{P}$ \\
\hline
15.0\,m/s & 35.0\,m/s & 100.0\,m & 80.0\,m  & 215.330\,m & 6 & 6 & 50.0\,m & 420\,m\\
\hline
\hline
\end{tabular}
\label{table:params}
\end{table}

Figure~\ref{fig:sim1a} depicts the simulation setup for Case-1, showing the sequence of events for a single outgoing UAV (in black) rejoining the main lane from the loiter circle. At $t=0\,\mathrm{s}$, multiple UAVs are initialized in the main lane with different separation distances. When the outgoing UAV reaches the reinsertion point of the loiter circle at $t=12.32\,\mathrm{s}$, the desired velocity of the UAV is computed and a suitable position within the green patch is determined for reinsertion. At $t=30.22\,\mathrm{s}$, the outgoing UAV enters the transit link path, and by $t=37.42\,\mathrm{s}$, it merges to the main lane. At $t = 38.2\,\mathrm{s}$, it can be observed that the UAV's are maintaining an appropriate safe distance from the other UAV and moving forward.\\

The velocity, lateral acceleration, and flight-angle profiles of the outgoing UAV, along with the separation distance between main lane UAVs, are shown in Fig.~\ref{fig:sim1b}.

\begin{figure}[H]
    \centering

    \begin{subfigure}{1.0\textwidth}
        \centering
        \includegraphics[width=\textwidth]{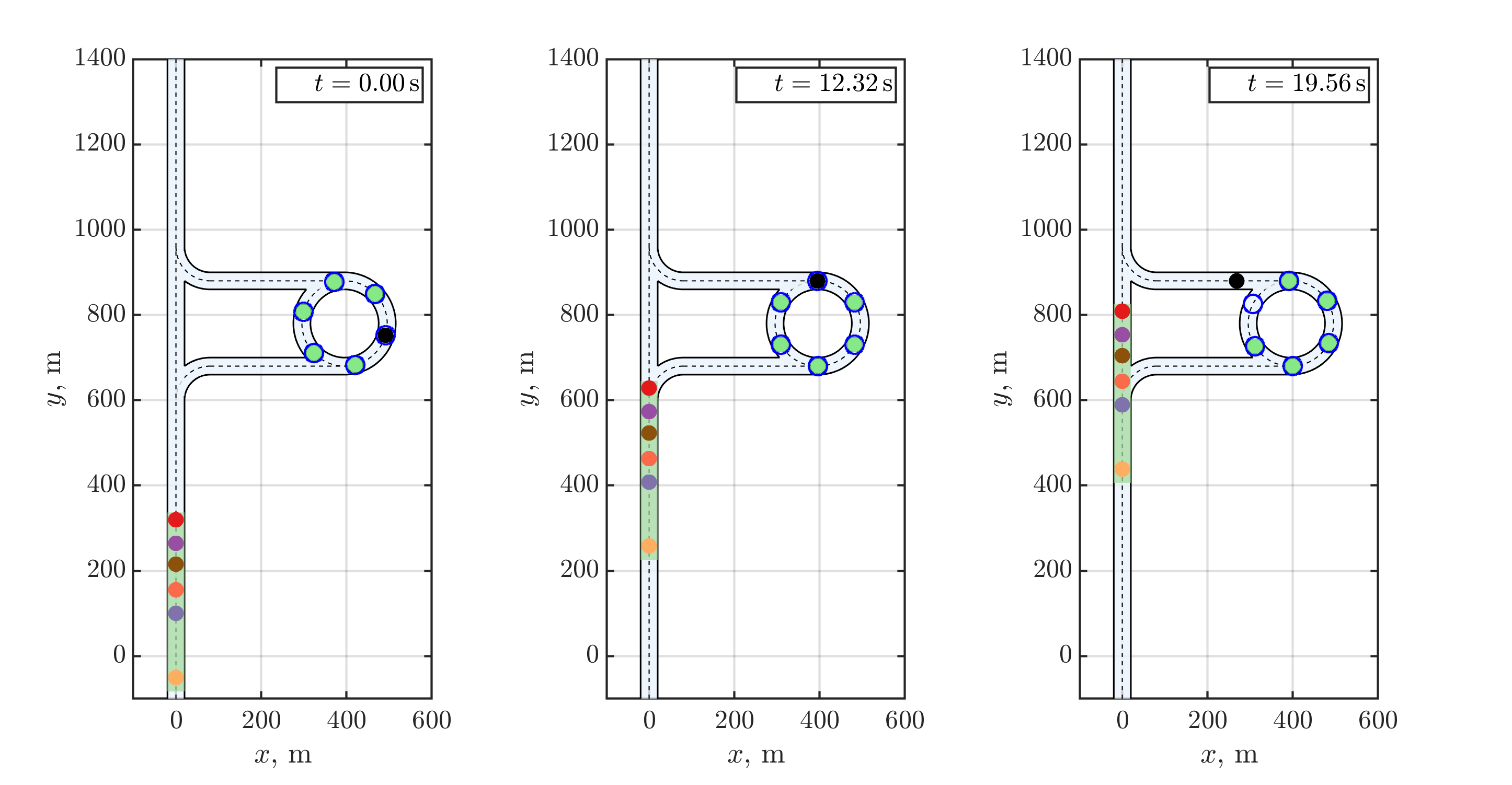}
    \end{subfigure}

    \begin{subfigure}{1.0\textwidth}
        \centering
        \includegraphics[width=\textwidth]{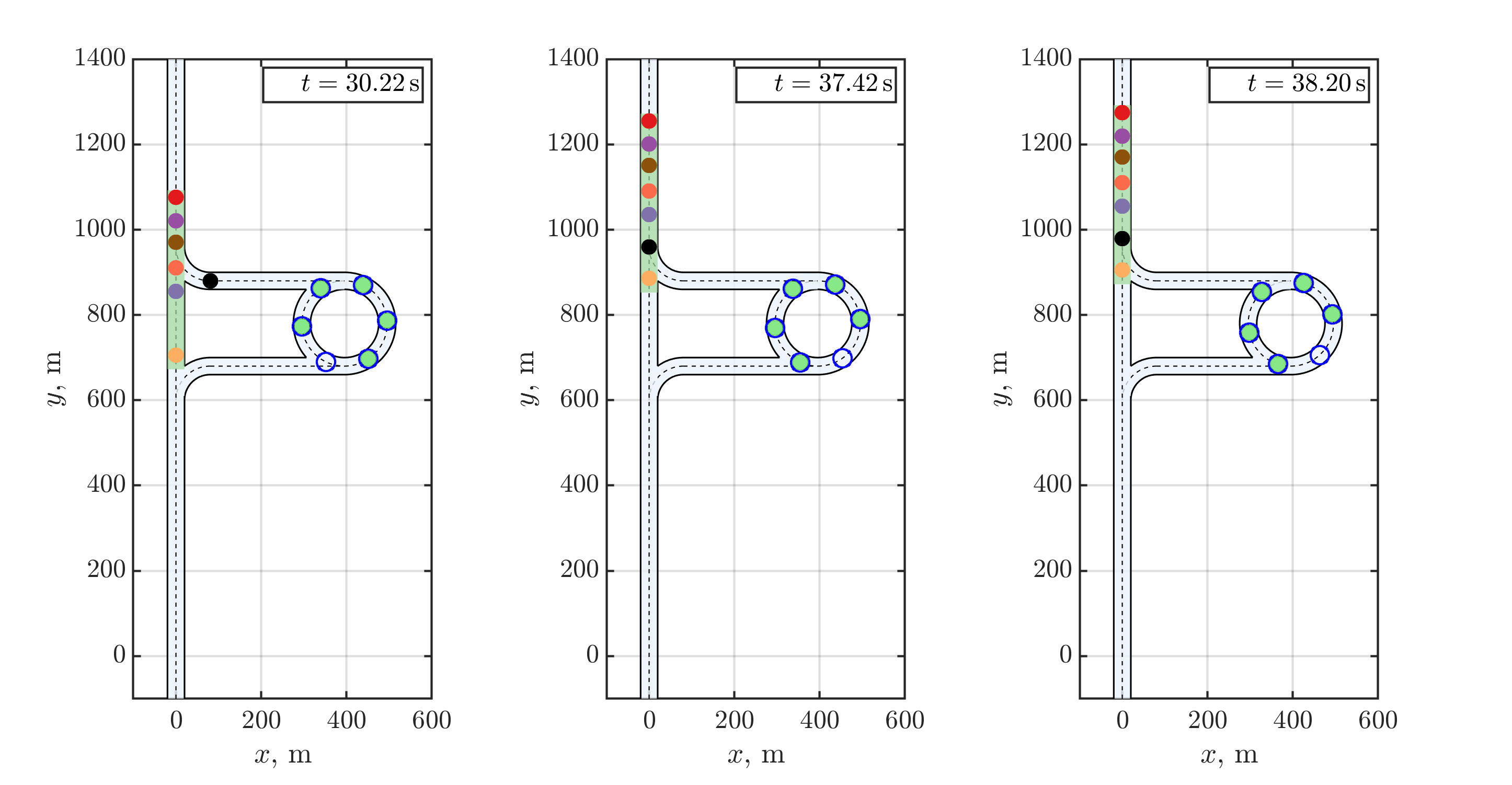}
    \end{subfigure}

    \caption{Case-1 Trajectory plots}
    \label{fig:sim1a}
\end{figure}

\begin{figure}[H]
    \centering
    \includegraphics[width=\textwidth]{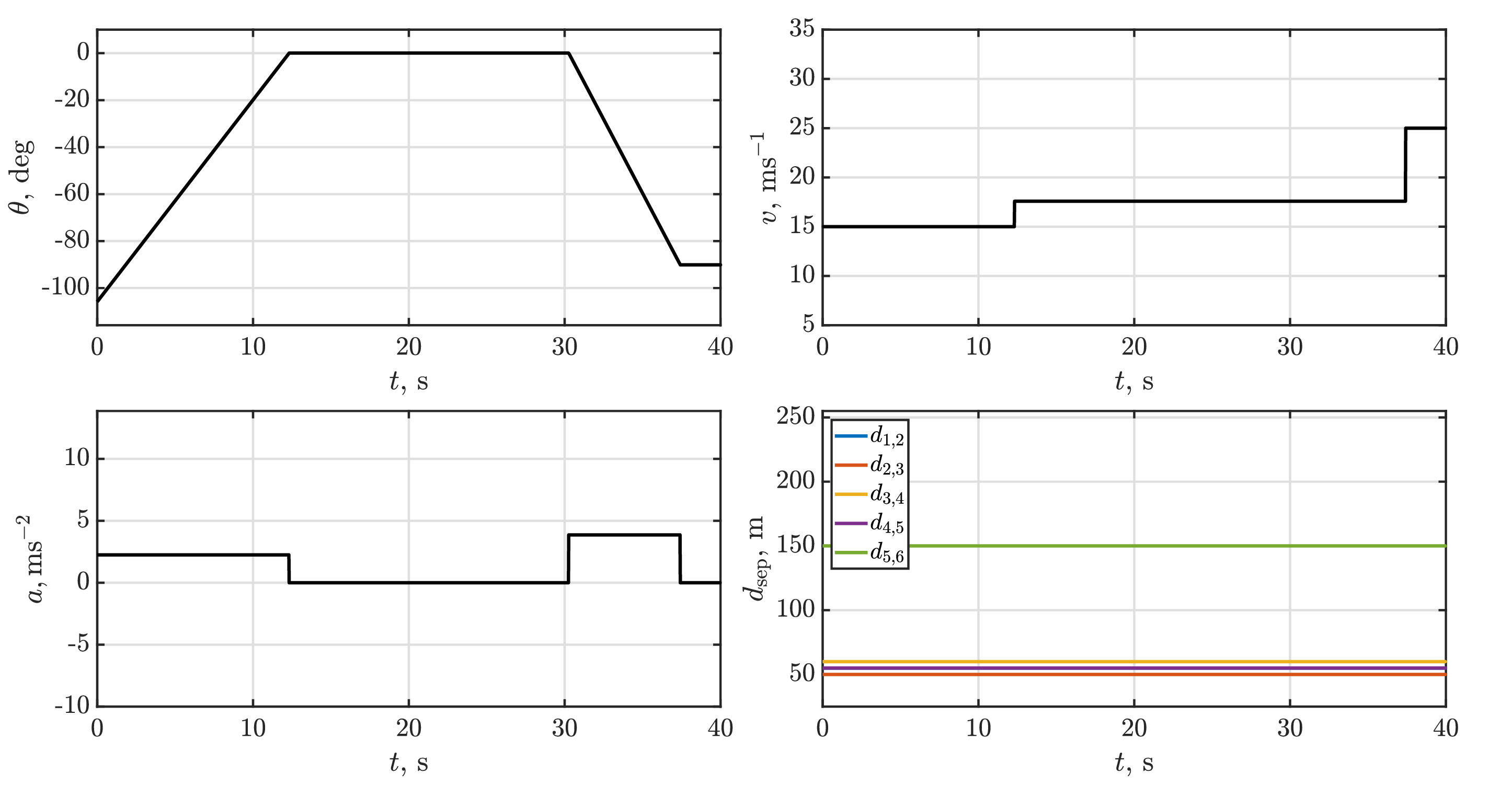}
    \caption{Case-1 Analysis plots}
    \label{fig:sim1b}

\end{figure}

Figure~\ref{fig:sim2a} depicts the simulation setup for Case-2. In this scenario, the UAV are initialized such there will no feasible location initially to get inserted. At $t=0\,\mathrm{s}$, multiple UAVs are initialized in the main lane. At $t=12.32\,\mathrm{s}$, the outgoing UAV reaches point D and finding no immediate feasible location to get inserted in green patch.The algorithm re-arrange the UAV's in main lane to create a feasible location for outgoing UAV. By $t=19.56\,\mathrm{s}$, the main lane UAV adjust it's UAV. At $t=30.2\,\mathrm{s}$, enters the transit link, and by $t=37.42\,\mathrm{s}$, it merges to the main lane. This scenario’s analysis plots are presented in Fig.~\ref{fig:sim2b}.

\begin{figure}[H]
    \centering
    \includegraphics[width=\textwidth]{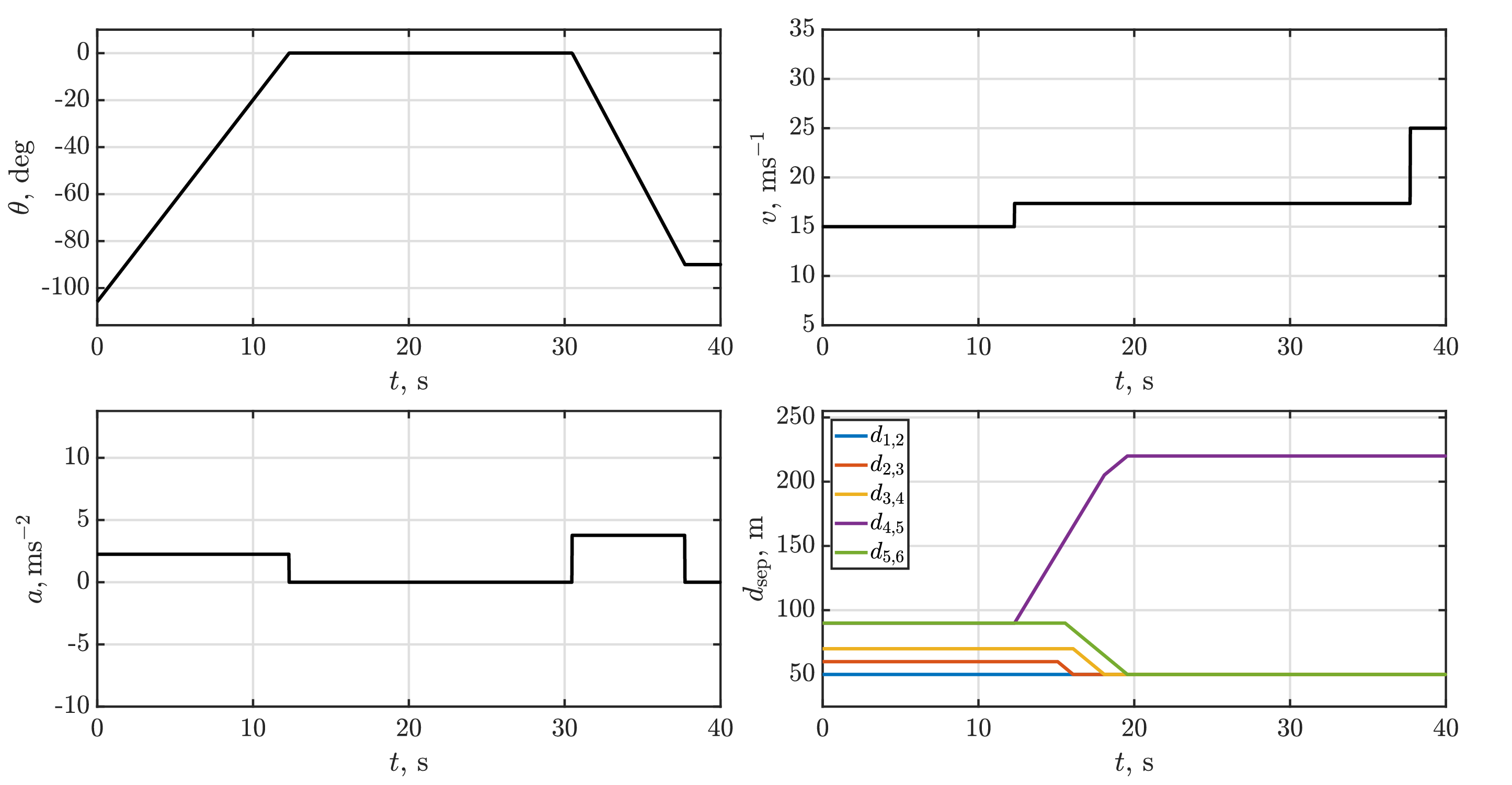}
    \caption{Case-2 Analysis plots}
    \label{fig:sim2b}
\end{figure}

\begin{figure}[H]
    \centering

    \begin{subfigure}{1.0\textwidth}
        \centering
        \includegraphics[width=\textwidth]{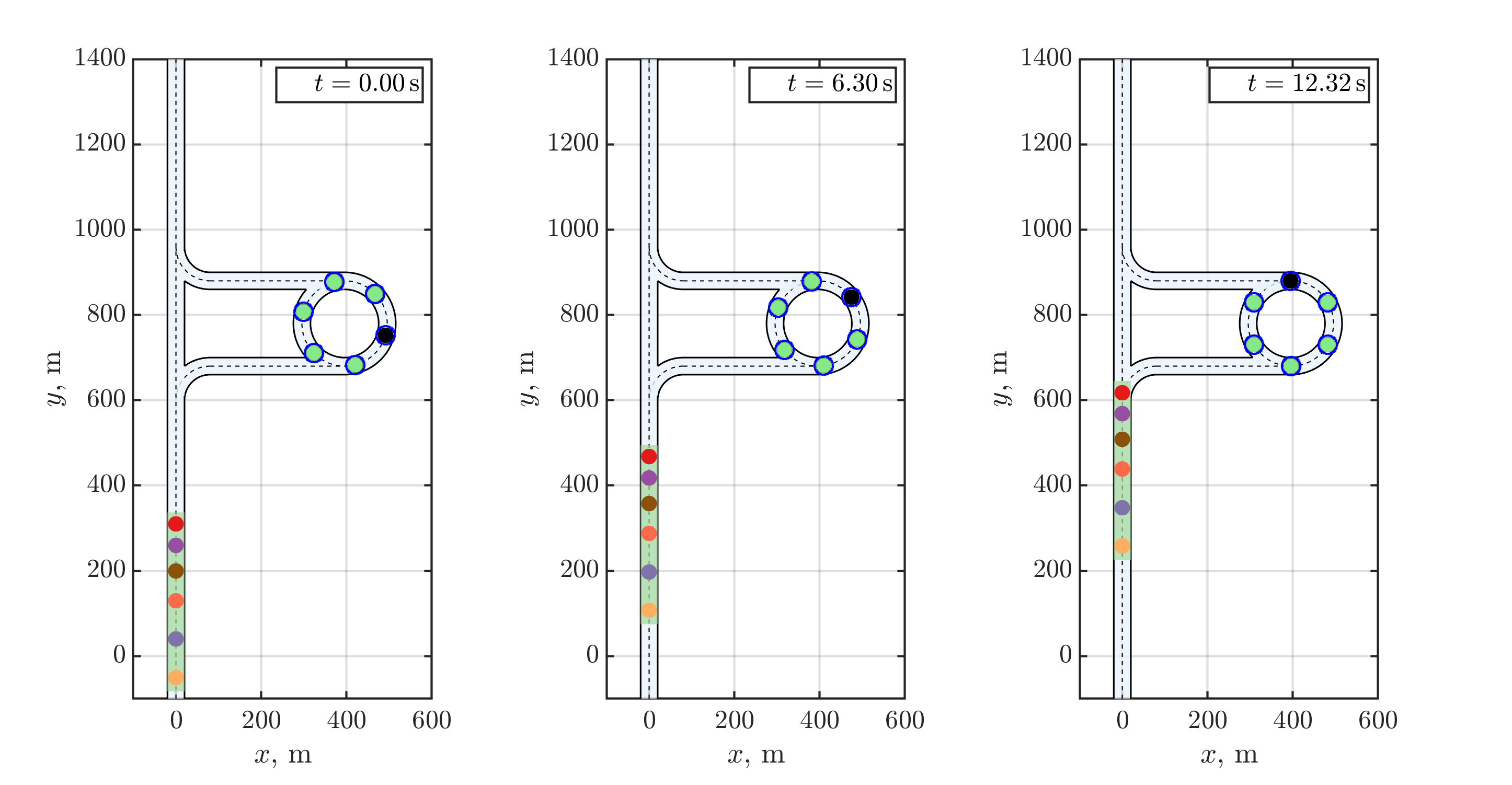}
    \end{subfigure}

    \begin{subfigure}{1.0\textwidth}
        \centering
        \includegraphics[width=\textwidth]{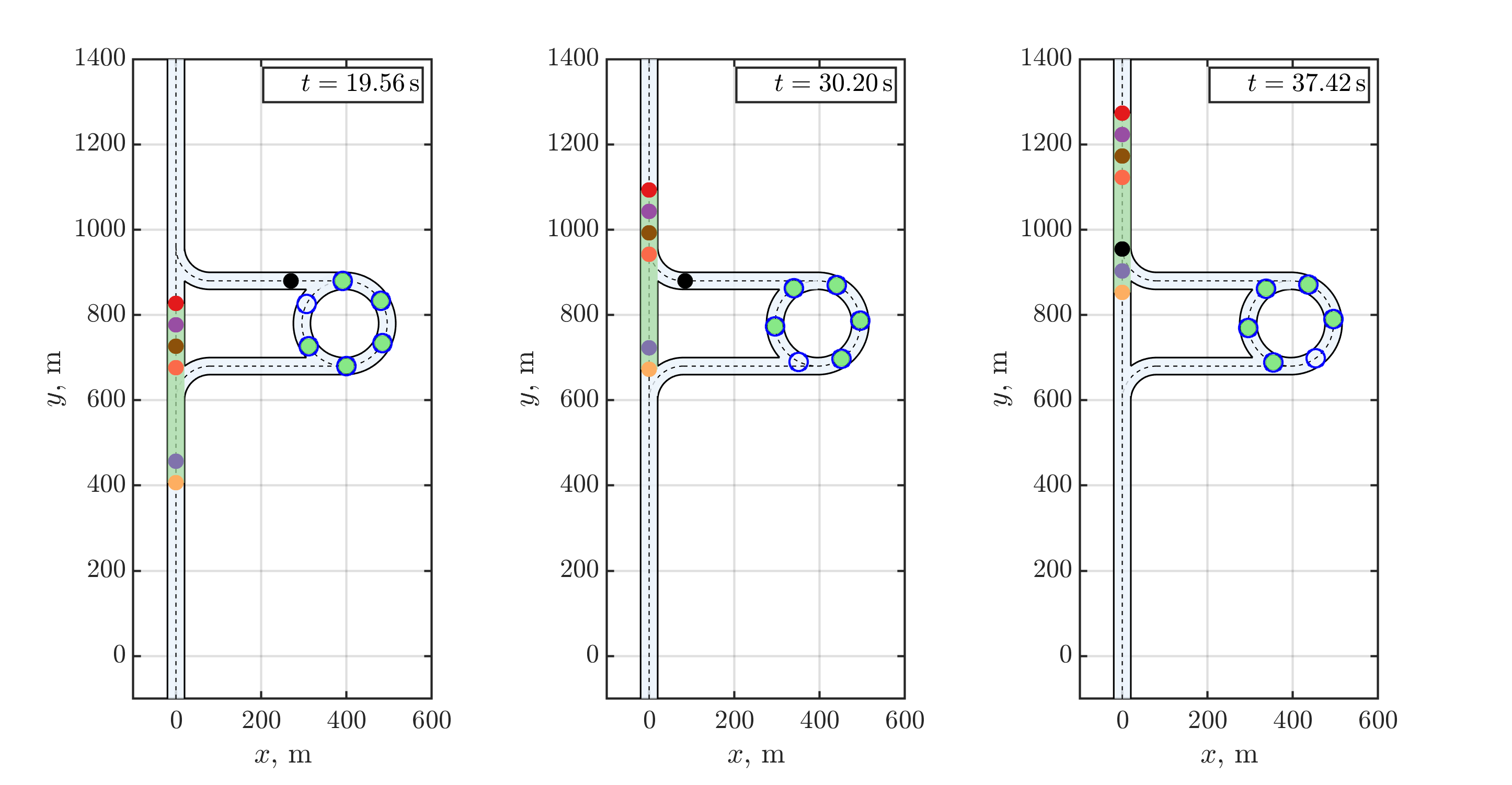}
    \end{subfigure}

    \caption{Case-2 Trajectory plots}
    \label{fig:sim2a}
\end{figure}

\section{Conclusion}
This paper presents a reinsertion guidance algorithm that enables a UAV to transition safely from a loiter circle back into the main lane. The algorithm generates commands for both the outgoing UAV and the main-lane UAVs so that the reinsertion can be carried out smoothly and safely. The speeds of the main-lane UAVs are adjusted only when necessary to accommodate the reinserting UAV. In addition, the design of the corridor for safe insertion—while ensuring the required safety distance—is discussed. The effectiveness of the proposed guidance strategy is demonstrated through simulation results.

\bibliography{sample}

\end{document}